\documentclass{article}

\usepackage[preprint, nonatbib]{nips_2018}

\usepackage{cite}
\usepackage[utf8]{inputenc} % allow utf-8 input
\usepackage[T1]{fontenc}    % use 8-bit T1 fonts
\usepackage{hyperref}       % hyperlinks
\usepackage{url}            % simple URL typesetting
\usepackage{booktabs}       % professional-quality tables
\usepackage{amsfonts}       % blackboard math symbols
\usepackage{amsthm}         % non numbered proof statements
\usepackage{amsmath}        % for multiline equations
\usepackage{nicefrac}       % compact symbols for 1/2, etc.
\usepackage{microtype}      % microtypography
\usepackage{tikz}           % for creating figures
\usepackage[outdir=./]{epstopdf}       % for creating figures from eps
\usepackage{epspdfconversion}
\usepackage{wrapfig}        % for word wrapping

\newtheorem{theorem}{Theorem}
\newtheorem{lemma}{Lemma}
\newtheorem{definition}{Definition}
\newtheorem{corollary}{Corollary}
\newcommand{\repeatthanks}{\textsuperscript{\thefootnote}}

\title{Universal Lipschitz Approximation in Bounded Depth Neural Networks}

\author{
  Jeremy E.J. Cohen\thanks{equal contribution}\\
  Perspecta Labs\\
  Basking Ridge, NJ\\
  \texttt{jcohen@perspctalabs.com} \\
  \And
  Todd Huster\repeatthanks\\
  Perspecta Labs\\
  Basking Ridge, NJ\\
  \texttt{thuster@perspectalabs.com} \\
  \And
  Ra Cohen\\
  Perspecta Labs\\
  Basking Ridge, NJ\\
  \texttt{rcohen@perspectalabs.com} \\
}

\begin{document}

\maketitle

\begin{abstract}
  Adversarial attacks against machine learning models are a rather hefty obstacle to our increasing reliance on these models.
  Due to this, provably robust (certified) machine learning models are a major topic of interest.
  Lipschitz continuous models present a promising approach to solving this problem.
  By leveraging the expressive power of a variant of neural networks which maintain low Lipschitz constants, we prove that three layer neural networks using the FullSort activation function are Universal Lipschitz function Approximators (ULAs). 
  This both explains experimental results and paves the way for the creation of better certified models going forward.
  We conclude by presenting experimental results that suggest that ULAs are a not just a novelty, but a competitive approach to providing certified classifiers, using these results to motivate several potential topics of further research.
\end{abstract}

\section{Introduction}

%Start by introducing Lipschitz approximation colloquoly
%Then introduce the applications focusing on Adversarial Machine Learning
%Focus on image classification problems

Lipschitz functions are continuous functions whose rate of change has a global upper bound. Such functions regularly occur in the natural world. For a simple example, the speed of a car over the course of a drive is continuous with bounded derivative; traits that imply the Lipschitz property. Suppose that we wish to create an approximation to this function. Traditional machine learning approaches, such as a neural network may attain good approximations to the original function, but they will not necessarily inherit the same upper bound on their rate of change. We will discuss Lipschitz approximation, the study of approximating a Lipschitz function by other functions with the same bound. Finding a class of functions that can approximate any Lipschitz function over a particular domain and range (a Universal Lipschitz Approximator) is a challenging problem. We expand on the work of Anil and Lucas, who found the first Universal Lipschitz Approximator, \cite{Anil2018} by finding more a practical class of functions that form ULAs.

The study of ULAs is primarily motivated by the study of adversarial machine learning. Neural networks have proven themselves to be highly capable in solving many challenging tasks \cite{Krizhevsky2012, Collobert2011}. Unfortunately, it has also been shown that with an imperceptible change to an existing input, they can fail catastrophically \cite{Szegedy2014, Goodfellow2016}. Compounding the problem, several concrete methods already exist to create these malicious distortions or "adversarial examples" \cite{Madry2017, Goodfellow2014, Carlini2016}. Compounding the problesm, empirically effective defenses to model robustness have often been broken soon after their invention \cite{Carlini2016, Athalye2018}. This is due to the fundamental fact that attacks can always take advantage of large gradients that traditional neural networks must have in order to represent their target functions \cite{Huster2018-pprint}. Lipschitz approximation avoids this problem by design.

While adversarial machine learning inspired the creation of Lipschitz function approximation, it is worth mentioning that another application has already been discovered. ULAs can be used to optimize over the set of all $1$-Lipschitz functions. In certain cases, this is the dual formulation of the Wasserstein metric, also known as the earthmover metric, which has been shown to be more stable than KL divergence in the training of GANs \cite{Anil2018}. We hypothesize that the fairly abstract mathematical concept of ULAs will find many other applications.

Because it is the main motivation for ULAs, and a well studied field with many other approaches to compare to, we will focus on the problem of adversarial machine learning. In particular, as it provides us with the richest context, we will operate on the task of image classification.

The rest of this paper is broken up as follows:

\begin{itemize}
    \item Section \ref{sec:related}: A brief survey of work on defenses against adversarial machine learning
    \item Section \ref{sec:ULA}: A discussion on Universal Lipschitz Approximators
    \item Section \ref{sec:2layer}: Evidence suggesting that two-layer FullSort networks do not appear to form ULAs
    \item Section \ref{sec:3layer}: A proof that three-layer FullSort networks do form ULAs
    \item Section \ref{sec:exp}: Experiments demonstrating the benefits of applying ULAs to classification problems
\end{itemize}

\section{Related Work}
\label{sec:related}

Many adversarial machine learning defenses have been proposed. There are two general categories of defenses, uncertified and certified, the difference being whether a defense is empirical or guaranteed. The first defenses to arrive on the scene were uncertified, empirical defenses based on the concept using adversarial training data to inoculate a model against specific types of adversarial examples \cite{Goodfellow2014, Goodfellow2016, Madry2017}. Unfortunately, similar to the flu, it is difficult and costly to inoculate against all types of adversarial examples; many of these approaches were eventually found weak to sophisticated attacks \cite{Athalye2018, Carlini2016}. In an effort to escape the cycle of cat and mouse, researchers began to instead focus on certified approaches. Instead of adding adversarial examples to the training data, they found ways to prove that, for a class of model, the following is true: around a given input, there is a region of the input space where the output of the classifier does not change. The larger the region, the stronger the certificate. 

Methods for generating and verifying these models vary quite wildly. Some of the first approaches evolved from research on the verification of properties of neural networks \cite{Katz2017, Cheng2017}. Leveraging variations on mixed integer linear programming, they were able to precisely compute certifications for smaller models. While the bounds that these methods provide are tight, the computational costs balloon exponentially as model size increases. 

The outer polytope method involves solving the dual of an optimization problem over a convex polytope. Wong showed that by considering a set of $L_\infty$ norm perturbations about an input ($B_\epsilon(x)$) and computing a convex outer bound about its image through the neural network ($A \supseteq f[B_\epsilon(x)]$), they could very successfully penalize a models susceptibility to perturbation \cite{Wong2017, Wong2018}. Perhaps motivated by the criticism about the computational complexity of the approach in their first paper, their second outlines a less computationally intensive method. This approach gives the strongest currently known $L_\infty$ results. As for the $L_2$ norm, randomized smoothing provides impressive stochastic certificates on $L_2$ perturbations in the input space by smoothing model outputs using a Gaussian kernel \cite{Cohen2019}. 

Universal Lipschitz Approximation is not the first Lipschitz based approach to be suggested for defending machine learning classifiers. By enforcing an orthonormality constraint on their weight matrices, Parseval networks have achieved limited success compared to early adversarial training methods \cite{Cisse2017}. Lipschitz margin training fared better than Parseval networks by using a novel differentiable way of computing an upper bound on the Lipschitz constant in order to penalize it during training \cite{Tsuzuku2018}. Most impressively, Raghunathan proposed a network architecture that can be trained to produce high quality classification results and certificates quite quickly \cite{Raghunathan2018}. Their approach, a semi-definite program that bounds the model Lipschitz constant, limited them to operating only on two layer networks. However, all traditional networks following this approach are greatly limited; Huster showed that no traditional neural network such as those using the ReLU activation function can act as a universal Lipschitz approximator. Traditional networks must in fact choose between either the expressive power necessary to approximate a function, or the Lipschitz condition \cite{Huster2018-pprint}. The search continued into non-traditional neural networks, and in 2018 Anil and Lucas showed that by changing out a standard monotonically increasing activation function for a sorting activation function, their networks, given arbitrary depth, form a ULA \cite{Anil2018}. 

\section{Universal Lipschitz Approxmiation}
\label{sec:ULA}

\subsection{Definitions}

Lipschitz functions are formally defined as follows:

\begin{definition}\label{def:Lipschitz}
Let a function $f$ be called $M$-Lipschitz continuous if there exists $M \geq 0$ such that

\begin{eqnarray}
\forall x_{1}, x_{2} \in X: d_{Y}(f(x_{1}),f(x_{2})) \leq M d_{X} (x_{1},x_{2})
\end{eqnarray}

where $d_{X}$ and $d_{Y}$ are the metrics associated with spaces $X$ and $Y$, respectively. We will call any $M$ for which this equation is true a Lipschitz constant of $f$. \end{definition}

Note that \emph{the} Lipschitz constant of $f$ is the minimum of all Lipschitz constants for $f$.

Intuitively, a Lipschitz function is, to a certain degree, smooth. Small changes in the input will produce measured changes in the output. For example, as we alluded to earlier, all differentiable functions with bounded derivative are Lipschitz continuous. Perhaps more importantly, given any finite dataset where different classes are separated in the input space by at least a  distance of $c$, there exists a Lipschitz function with Lipschitz constant $c/2$  that correctly classifies all points \cite{Huster2018-pprint}. Therefore, any classification problem can be restated in the structure of a Lipschitz function approximation problem.

\begin{definition}\label{def:lipappx}
Let $f: X \rightarrow Y$ be an $M$-Lipschitz continuous function, and let $\epsilon > 0$. We say that a function $g$ $\epsilon$-approximates $f$ if $M$ is a Lipschitz constant of $g$ and

\begin{equation}
\forall x \in X: d_{Y}(f(x), g(x)) < \epsilon
\end{equation}
\end{definition}

These problems can come in many shapes and sizes. We focus on the case where $X$ and $Y$ are both normed spaces. Note that if 
$f$ is $M$-Lipschitz, then $f/M$ is $1$-Lipschitz. This implies that should $g$ $\epsilon$-approximate $f$, $g/M$ must $\epsilon / M$
-approximate $f/M$. Thus, it suffices to consider the case where $M$ is $1$. Hence, given a specific normed domain and range, we wish
to be able to approximate any $1$-Lipschitz function between them. A universal Lipschitz approximator is able to do the following:

\begin{definition}
Fix norm spaces $X$ and $Y$. A set of functions, $G$ that map $X$ into $Y$ is called a Universal Lipschitz Approximator (ULA) if, given any function $f: X \rightarrow Y$ that is $1$-Lipschitz

\begin{equation}
\forall \epsilon > 0, \, \exists g \in G \text{ that } \epsilon \text{-approximates } f
\end{equation}
\end{definition}

There are a couple advantages to preforming Lipschitz approximation rather than standard approximation. Firstly the obvious, if the Lipschitz constant of a Lipschitz continuous model is low, its susceptibility to adversarial examples is reduced \cite{Gouk2018}. For example, suppose that we have a classifier $f: X \rightarrow \mathbb{R}^2$ that is $1$-Lipschitz and at point $x$ has value $f(x) = [1.5, 0.5]$. That is, it predicts that $x_0$ belongs to the first class by a margin of $1.5 - 0.5 = 1$. Then, because we know that $f$ is $1$-Lipschitz, we know that for all $\delta \in X$ where $||\delta||_{X} < 0.5||$, $f$ will predict that $x_0 + \delta$ belongs to the first class as well. In other words, if $g : X \rightarrow \mathbb{R}$ where $g(x) = f(x)_1 - f(x)_2$,

\begin{equation}
\begin{split}
g(x_0 + \delta) &= f(x_0 + \delta)_1 - f(x_0 + \delta)_2\\
                &< (f(x_0)_1 - 0.5) - (f(x_0)_2 + 0.5)\\
                &= 1.5 - 0.5 - 0.5 - 0.5\\
                &= 0\\
\end{split}
\end{equation}

. Thus, if the model has a small Lipschitz constant and is confident about its classification on a specific input, a slight perturbation to that input will not change the output of the model. The other major argument for Lipschitz approximation is based on some evidence that smoother models generalize better than those without smoothness guarantees \cite{Gouk2018, Oberman2018}.

Computing a Lipschitz constant of a neural network is not necessarily difficult. An SMT or MILP solver can be used to compute the minimal Lipschitz constant, but this process is computationally intensive and unlikely to be practical for non-trivial networks \cite{Katz2017, Cheng2017}. Another approach is to calculate an upper bound for the minimal Lipschitz constant. One method, proposed by Gouk \cite{Gouk2018}, consists of computing the function norm of the network by considering it layer by layer. Since composed functions inherit the product of their individual Lipschitz constants, given a network with weight matrices $\{W_i\}_{i=1}^k$ and activation function $\sigma$, a Lipschitz constant of the network can be computed by

\begin{equation}
M = \prod_{i=1}^k ||\sigma||_o \cdot ||W_i||_o
\end{equation}

for the operator norm $||\cdot||_o$, specifically chosen with respect to the p-norm in the domain and range $||\cdot||_p$ over which we want our classifier to be robust. For many popular choices of $\sigma$ such as ReLU, the value $||\sigma||_o$ is bounded above by (and sometimes even equal to) $1$. Computation time of $||W_i||_o$ depends on the operator norm, while some are NP-hard to compute, others, such as the operator norm induced by the $L_\infty$ norm take $nm$ time to compute where $n$ and $m$ are the number of rows and columns in the matrix.

Computing the Lipschitz constant is not sufficient for Lipschitz function approximation, a model is also needed. Maintaining expressive power of the model while also enforcing a Lipschitz bound is beyond the capabilities of traditional neural networks. Thus, these networks can only act as Lipschitz approximators on a very limited class of functions. A class of functions so limited that it does not even include the absolute value function \cite{Huster2018-pprint}! Anil and Lucas suggest that this is due to the norm-decreasing properties of these traditional monotonically increasing activation functions. They claim that norm-preserving non-linear activation functions can get around this hurdle, specifically suggesting the GroupSort activation function \cite{Anil2018}. 

For the rest of this paper, we will be using the $L_\infty$ norm as well as the $L_2$ norm in some of our empirical evaluations.

\subsection{GroupSort Networks}

The GroupSort activation function breaks the pre-activation neurons into groups of a specified group size and sorts them by their value. When group size is two, GroupSort is equivalent to the Orthogonal Projection Linear Unit (OPLU) \cite{Chernodub2016}: for each pair of neurons, the lowest value is assigned to the first output and the highest value is assigned to the second. When group size is equal to the number of neurons, it is a FullSort and the output is a sorted list of all the inputs.  

Anil and Lucas prove that using a set of 1-Lipschitz linear basis functions (i.e., a single norm-constrained linear layer) and an arbitrary number of OPLU operations on different combinations of previously computed functions, one can approximate \emph{any} $1$-Lipschitz function while maintaining the $1$-Lipschitz property of the model \cite{Anil2018}. Relating their result to the Universal Approximation theorem \cite{Cybenko1989}, both results guarantee that a class of neural networks can arbitrarily approximate a target class of functions. However, the Universal Approximation theorem guarantees that networks of \emph{two} layers can arbitrarily approximate all functions of interest, and as of yet, there has been no such proof that any type of neural networks of constant bounded depth can Lipschitz approximate all Lipschitz functions.

Experimentally, Anil and Lucas verified that GroupSort can be used as a drop-in replacement for ReLU activations with little-to-no impact on the model accuracy \cite{Anil2018}. It was also shown that two-layer networks with a FullSort activation were easily able to approximate some toy $1$-Lipschitz functions, while networks using OPLU needed several layers to achieve the same end. These experiments act as the motivation behind our search for bounded depth ULAs. 

\section{Two Layer ULAs}
\label{sec:2layer}

The universal approximation theorem states that any continuous function on the $n$
dimensional unit cube can be arbitrarily approximated by a two layer neural network
with a sigmoidal activation function \cite{Cybenko1989}. At first, one may expect a direct
Lipschitz analogue to this: all $1$-Lipschitz functions can be arbitrarily Lipschitz approximated by
a two layer neural network that uses a sorting activation function.

\begin{figure}[t]
    \includegraphics[width=.33\linewidth]{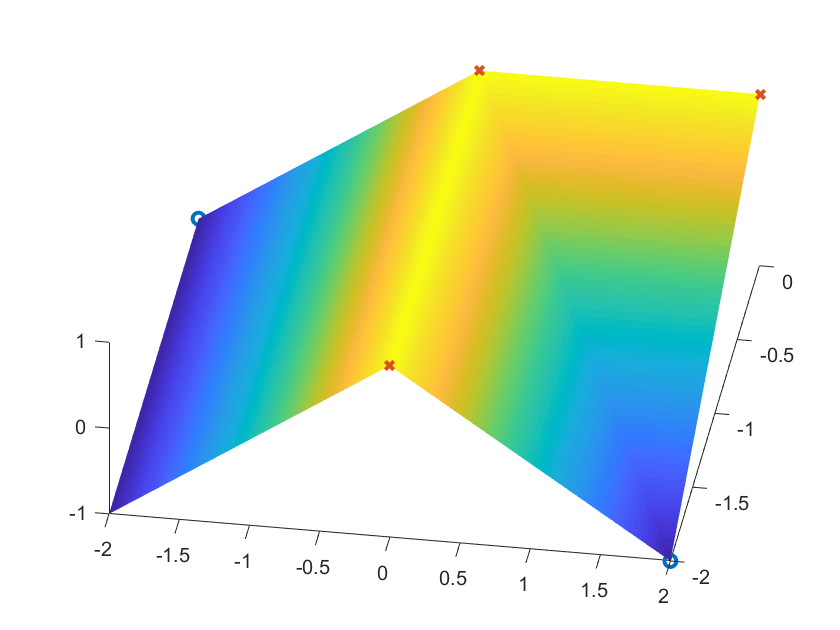}
    \includegraphics[width=.33\linewidth]{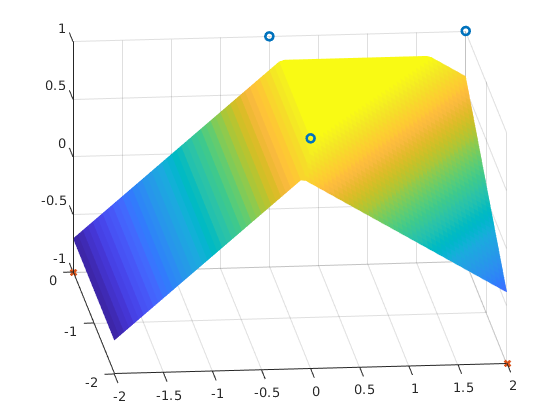}
    \includegraphics[width=.33\linewidth]{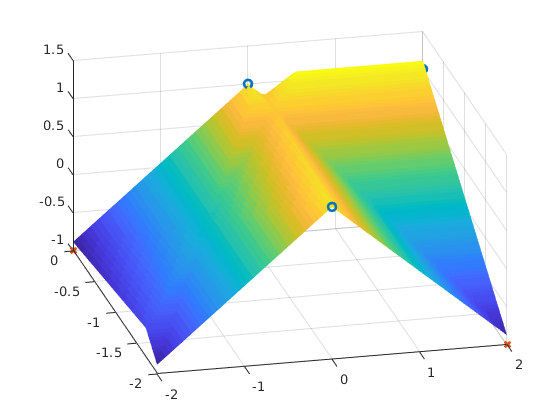}
    \caption{The half gable function (left), our best two layer Lipschitz FullSort NN approximation (center), and a three layer Lipschitz FullSort NN approximation (right).}
    \label{fig:gable}
\end{figure}

In a sense that will become more technically clear in later proofs, FullSort supersedes
all GroupSort-$n$ style activation functions in expressive power. Hence, we will focus
on FullSort networks for the time being. Through experimentation, we found that FullSort
networks of two layers seem to fail to learn even simple two dimensional functions on which three layer networks succeed.
Equivalently trained architectures are compared in Figure \ref{fig:gable}.

We call the left-most Lipschitz function from Figure \ref{fig:gable} the half gable. It can either
be defined by the three hyperplanes that compose it, or the five points via which it
is the only 1-Lipschitz function that satisfies them. If we were to try to most simply construct
a two layer FullSort network that equals this function, we may design our first layer
as follows

\begin{equation}
    \label{eq:naive}
    FullSort(
    \begin{bmatrix}
        0 & 1 \\
        -1 & 0 \\
        1 & 0 
    \end{bmatrix}
    \begin{bmatrix}
        x \\
        y
    \end{bmatrix}
    +
    \begin{bmatrix}
        1 \\
        1 \\
        1
    \end{bmatrix}
    )
\end{equation}

\begin{wrapfigure}{L}{0.33\textwidth}
    \includegraphics[width=0.33\textwidth]{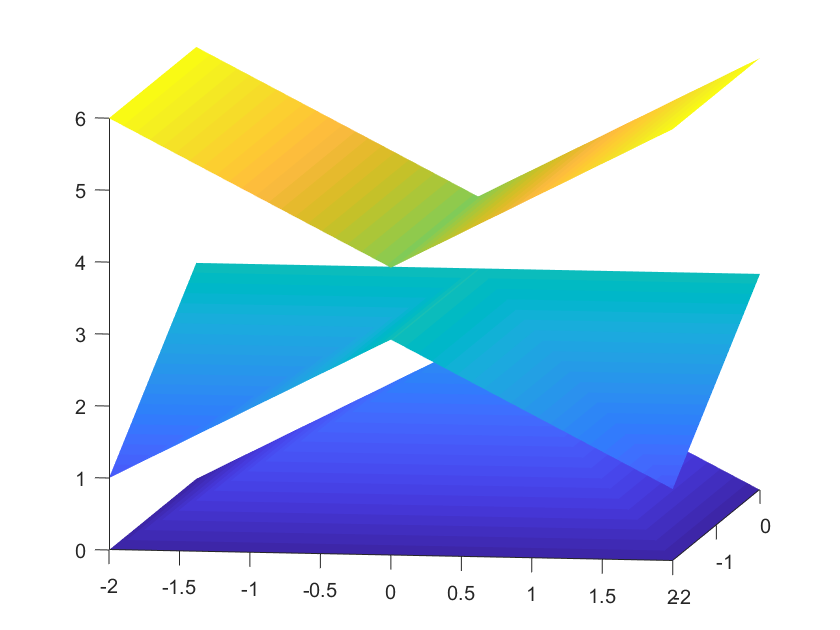}
    \caption{The three different possible outputs given the first layer in equation \ref{eq:naive} and a one-hot second layer.}
    \label{fig:sorts}
\end{wrapfigure}

. This produces a sorted vector of values that each correspond to the value of one of the
three hyperplanes which compose the half gable. Suppose the weight matrix of the second
layer is one of the following: $[0,0,1]$, $[0,1,0]$, or $[1,0,0]$. Then our network computes
one of the three functions displayed in Figure \ref{fig:sorts}, none of which are the half
gable. While this does not constitute a proof, it does suggest that if there exists a two layer
$1$-Lipschitz neural network that closely approximates this function, it will not be a simple
construction. We can compare this to the fairly simple three layer $1$-Lipschitz network which
is exactly equal to the half gable, and whose existence is proven in section \ref{sec:3layer}.

Another way that we can try to understand these results is to think about the fundamental
difference between two and three layer FullSort networks. With two layers, the network
only has one sort. This means that it computes some number of $1$-Lipschitz linear functions,
sorts them, and then takes a linear combination of all of the resultant values. The weights of this
linear combination can not change depending on the input. However, in three layer networks,
there are two sorts. This allows for the network to throw away information between the sorts, 
and to better adapt to its particular input, concepts that will be taken advantage of in Theorem \ref{thm:ULA}.

\section{Three Layer ULAs}
\label{sec:3layer}

In order to prove that three layer FullSort networks are ULAs, we need to first shift domains and define a couple terms. Every neural network with sort activations represents a Piecewise Linear Function. 

\begin{definition}[Piecewise Linear Functions (PWL)]\label{def:PWL}
A function $f:D \subseteq \mathbb{R}^n \rightarrow \mathbb{R}$ is a PWL if $\,\exists L \in \mathbb{N}$ distinct linear 
functions $f_l$ such that $\forall x \in D, \exists i \leq L$ such that $f(x) = f_i(x)$.
\end{definition}

We are interested in PWLs that are continuous. This sort of PWL has been studied quite a bit in the field of electrical
engineering \cite{Tarela1, Tarela2, Wang, Wilkinson}. Such studies have produced several useful representations
of continuous PWLs. We will lean heavily upon the following representation whose universality was proved in \cite{Tarela2} 
and which is notated similar to its presentation in \cite{Wang}.

\begin{lemma}[Lattice Representation]\label{lem:LR}
Let $f : D \subseteq \mathbb{R}^n \rightarrow \mathbb{R}$ be a continuous PWL function with $L$ distinct linear functions $f_l$.
Then $f$ can be represented by $(F, S)$ where $F = \{f_l\}$ and $S$ is a set of subsets of $F$ such that
\begin{equation}
\min_{s_i \in S} (\max_{f_j \in s_i} (f_j))
\end{equation}
\end{lemma}

It was shown in \cite{Tarela1} that the order of the $\max$ and $\min$ operations in this representation is arbitrary and that
any function represented with $\min$ then $\max$ can be represented as $f = \max_{s_i \in S} (\min_{f_j \in s_i} (f_j))$ for 
different $S$. We will refer to this other representation as the dual lattice representation.

In order to show universal approximation, we will instead show the equivalent condition of density in our target space.
To this end, we define what a lattice is, and present an adaptation of the Stone-Wirestrass theorem proved in \cite{Anil2018}.

\begin{definition}[Lattice]\label{def:lattice}
We say that a set of functions $L$ is a lattice if for any $f, g \in L$ we have $\max(f, g) \in L$ and
$\min(f, g) \in L$. ($\max$ and $\min$ are defined point-wise)
\end{definition}

\begin{lemma}[Restricted Stone-Weierstrass Theorem]\label{lem:SWT}
Suppose that $(X, d_X)$ is a compact metric space  with at least two points and $L$ is a lattice in $C_L(X, \mathbb{R})$ 
with the property, known as separating points, that for any two distinct elements $x,y \in X$ and any two real numbers $a$ 
and $b$ such that $|a-b| \leq d_X(x,y)$ there exists a function $f \in L$ such that $f(x) = a$ and $f(y) = b$. Then $L$ is dense 
in $C_L(X,\mathbb{R})$.
\end{lemma}

Theorem \ref{thm:appx} will show that a  certain subset of continuous PWLs are dense in the space of Lipschitz functions from 
$D$ into $\mathbb{R}$. Theorem \ref{thm:ULA} will show that this particular subset of PWLs can be represented by three layer $1$-Lipschitz FullSort neural networks.

\begin{theorem}[Continuous PWLs of fixed slope are ULAs]\label{thm:appx}
Let $G$ be the set of continuous PWLs defined over a compact subset $D\subset \mathbb{R}^n$ imbued with the $L_\infty$ norm
such that if $f \in G$, the set of component linear functions $F$ all satisfy the property that $||f_j||_{\infty} = 1$. 
$G$ is dense in the set of 1 Lipschitz functions from $D$ into $\mathbb{R}$, $C_L(D, \mathbb{R})$.
\end{theorem}

\begin{proof}
Suppose $f \in G$ has lattice representation $(F, S)$, and $\hat{f} \in G$ has lattice representation $(\hat{F}, \hat{S})$.
Define $\bar{F} = F\cup \hat{F}$ and $\bar{S} = S \cup \hat{S}$. Now see that because $A = B \cup C$ implies $\min_{a\in A} = 
\min(\min_{b\in B}, \min_{c\in C})$ we have

\begin{equation}
\begin{split}
\bar{f} &= \min_{\bar{s_i} \in \bar{S}}(\max_{\bar{f_j} \in \bar{s_i}}) (\bar{f_j}))\\
        &= \min(\min_{s_i \in S}(\max_{f_j \in s_i}(f_j)), \min_{\hat{s_i} \in \hat{S}}(\max_{\hat{f_j} \in \hat{s_i}}) (\hat{f_j})))\\
        &= \min(f, \hat{f})\\
\end{split}
\end{equation}

. Because $\bar{F}$ is a union of functions with norm $1$, it itself is a set of functions with norm $1$. Thus, $\bar{f} \in G$
and therefore $G$ is closed under $\min$. If we start with the dual lattice representations of $f$ and $\hat{f}$, the same
calculation shows that $G$ is closed under $\max$. Thus, $G$ is a lattice.

It remains to be shown that $G$ has the point separation property required by Lemma \ref{lem:SWT}. The idea here is simple. Define 
two linear functions that attempt to bring the two points together and have gradient norm $1$. These two hyperplanes will
each achieve at least one of the points, and be less than or equal to the target value at the other point. They must
intersect, so the maximum of the two will be a continuous PWL in $G$ that separates the points. See figure \ref{fig:th1}.

%Figure 1. Tries to display the concept outlined in the above paragraph
\begin{figure}[t]
\centering
\begin{tikzpicture}[scale=4]
\draw (0,1) -- (0,0) -- (1,0) -- (1,1);

\draw [fill=black] (0.2, 0.6) circle [radius=0.015];
\node at (0.175, 0.5) {(x, a)};
\draw [fill=black] (0.8, 0.4) circle [radius=0.015];
\node at (0.825, 0.3) {(y, b)};

\draw [dashed] (-0.1, 0.9) -- (0.9, -0.1);
\draw [dashed] (1.1, 0.7) -- (0.3, -0.1);

\draw (0, 0.8) -- (0.6, 0.2);
\draw (0.6, 0.2) -- (1, 0.6);

\node at (0,-0.1) {(0,0)};
\node at (1,-0.1) {(1,0)};
\end{tikzpicture}
\caption{The one dimensional unit interval version of the construction in Theorem \ref{thm:appx}}
\label{fig:th1}
\end{figure}
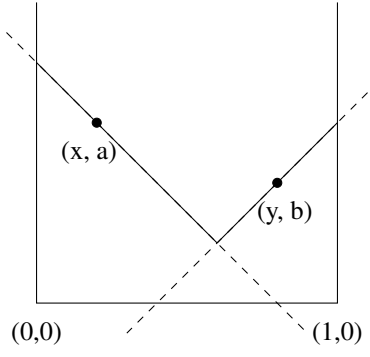

Fix $x, y \in D$ and $a, b \in \mathbb{R}$ such that $|a-b| \leq ||x-y||_\infty$. Without loss of generality, 
suppose that $a \geq b$. Now define $C = (x-y)/||x-y||_\infty$ and

\begin{equation}
f_0(z) = C \cdot z + a - C \cdot x \, , \, f_1(z) = - C \cdot z + b - C \cdot y
\end{equation}

Notice that $f_0(x) = C \cdot x + a - C \cdot x = a$. Similarly, $f_1(y) = b$. Also, as $||C||_\infty = 1$, $f_0, f_1 
\in G$.

Now consider $f_0(y)$,
\begin{equation}
\begin{split}
f_0(y) &= C \cdot y + a - C \cdot x\\
       &= C \cdot (y-x) + a\\
       &= - ||x-y||^2_2/||x-y||_\infty + a\\
    &\leq - ||x-y||^2_\infty/||x-y||_\infty + a\\
       &= a - ||x-y||_\infty\\
    &\leq a - |a - b|\\
       &= a - a + b = b
\end{split}
\end{equation}
because $||z||_2 \geq ||z||_\infty$ and $a \geq b$. An analogous calculation shows that, $f_1(x) \leq a$.

Thus, $\max(f_0, f_1)(x) = a$ and $\max(f_0, f_1)(y) = b$. Because $f_0, f_1 \in G, \max(f_0, f_1) \in G$.
Therefore, $G$ separates points as required in Lemma \ref{lem:SWT}, and $G$ is dense in $C_L(D, \mathbb{R})$.
\end{proof}

Now that we know that these functions are ULAs, it remains to be shown that constant-layer Lipschitz-constrained FullSort
networks are capable of computing them. The following construction is informed by the structure of the lattice representation.

\begin{theorem}[Universal Approximation of Lipschitz functions with Lipschitz networks]\label{thm:ULA}
For all $f \in G$ there exists a 3 layer neural network using the FullSort activation function that equals $f$. Furthermore,
the weight matrices of this network, $W_0$, $W_1$, and $W_2$ satisfy the property that $||W_i||_\infty = 1$.
\end{theorem}

\begin{proof}
For our use, suppose that FullSort produces an ascending vector, i.e. $i>j$ implies $FullSort(x)_i \geq FullSort(x)_j$.

Fix $f \in G$ and let $(F, S)$ be its lattice representation. Let each $f_j$ be represented by $f_j(z) = \nabla f_j\cdot z + c_j$.
Remember that because $f \in G$ we know that $||\nabla f_j||_1 = 1$. Define the sequence of running sums $\{\sigma_i\}_{i=1}^{|S|}$
by $\sigma_1 = |s_1|$ and $\sigma_{i} = sigma_{i-1} + |s_i|$. So that $\sigma_k$ corresponds to the number of elements in all 
$s_i$'s where $i \leq k$.

Let $\{W_{0,i}\}_{i=1}^{|S|}$ and $\{B_{0,i}\}_{i=1}^{|S|}$ be defined by stacking the rows $\nabla f_j$ for each $j \in s_i$ to create
$W_{0,i}$, and by stacking the associated $c_j$ to define $B_{0,i}$. Then for each $z\in D$, $W_{0,i}z + B_{0,i}$ is the stacked values
of each hyperplane $f_j \in s_i$ at the point $z$.

Because $|F|$ is finite and $D$ is compact there exists $\alpha \in \mathbb{R}^+$ such that for all $z \in D$ and $f_j \in F$,
$|f_j(z)| < \alpha$.

Now define a $n\times \sigma_{|S|}$ matrix $W_0$ by stacking each $W_{0,i}$, and define $B_0$ by stacking the arrays
$B_{0,i} + ([2i\alpha]_{k=1}^{|s_i|})^t$. We will hereby refer to the added $2i\alpha$ term as the separation term. 
Let $e_i$ be the standard basis vectors of $\mathbb{R}^{|S|}$. Then, we can define $W_1$ to be a $\sigma_{|S|} \times |S|$ 
matrix with the $i$th row equal to $e_{\sigma_i}$, and $B_1 = [-2i\alpha]_{i=1}^{|S|}$. Finally define $W_2$ to be a
$|S| \times 1$ matrix where the first element is $1$ and the rest are $0$s.

We claim that the neural network, $g$, defined by $W_0$, $B_0$, $W_1$, $B_1$, and $W_2$, with the $FullSort$ activation function
equals $f$.

Let $z \in D$. Then

\begin{equation}
\begin{split}
g(z) &= W_2 FullSort(W_1FullSort(W_0z + B_0) + B_1) \\
     &= \min(W_1FullSort([f_j(z)]_{f_j\in s_i, s_i \in S}^t + ([2i\alpha]_{k=1}^{|s_i|})^t) + B_1) \\
     &= \min(\{\max_{f_j \in s_i}(f_j(z) + 2i\alpha) - 2i\alpha\}_{i=1}^{|S|}) \\
     &= \min_{s_i \in S}(\max_{f_j \in s_i}(f_j(z))) \\
     &= f(z) \\
\end{split}
\end{equation}

. Therefore, $g = f$ and for any $f \in G$ there exists a neural network $g$ that equals $f$. We can see that because for all
$f_j, ||\nabla f_j||_1 = 1$, $||W_0||_\infty = 1$ (remember that for matrices, the infinity norm is the maximum absolute row sum).
 Similarly, $||W_1||_\infty = ||W_2||_\infty = 1||$. Thus, these neural networks satisfy the statement of the theorem.
\end{proof}

It should be noted that a slightly stronger version of Theorem \ref{thm:ULA} is possible. Because we use a trick to turn Fullsort into a
sort of GroupSort on the first layer, we could just pad out the excess rows in the sorting block with values either impossibly low or high.
Therefore, it is possible to show that the union of all GroupSort-$n$ networks for each natural number $n$ form a ULA.

Together, Theorem \ref{thm:appx} and Theorem \ref{thm:ULA} imply the following:

\begin{corollary}
Three layer $FullSort$ neural networks with weight matrices that have infinity norm $1$ are dense in $C_L(X,\mathbb{R})$.
\end{corollary}

\section{Experiments}
\label{sec:exp}

We now experimentally validate shallow GroupSort networks against other other Lipschitz-based certification approaches. We compare fully connected networks using both the GroupSort and ReLU activation functions.  We penalized the Lipschitz constant by directly adding a Lipschitz penalty term to the typical cross entropy loss function, H: 

\begin{eqnarray}
L = H(f(x),y) + \lambda \prod_{i=1}^{k}\|W_i\|_p
\end{eqnarray}

where $\{W_i \mid i=1..k\}$ are the weight matrices of $f$.  

This penalization approach is in contrast to the norm-enforcement technique used by Anil and Lucas \cite{Anil2018}, which we found less effective. Rather, if we were using ReLU activations, our approach would be very similar to Lipschitz margin training (LMT) \cite{Tsuzuku2018}. We explored the hyperparameter space, varying over number and width of layers, learning rate, value of $\lambda$, and group size (for GroupSort). We selected our final models based on training data performance, which closely matched test performance. 

We also trained two layer networks using the SDP method in \cite{Raghunathan2018}, searching over layer width and penalty strength for the optimal hyperparameters.  Even though this method is specific to two-layer networks under the $L_{\infty}$ norm, prior to our current work, this method gave the best certified bounds based on the Lipschitz constant. 

For a given Lipschitz constant, GroupSort networks are more expressive than either standard ReLU networks or two layer SDP networks; as shown in \cite{Huster2018-pprint}, these methods cannot closely approximate some simple functions while maintaining tight bounds on the Lipschitz constant.

\begin{figure}[t]
    \includegraphics[width=.33\linewidth]{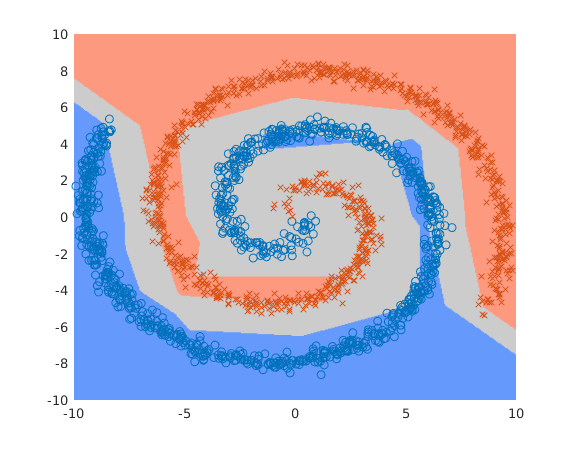}
    \includegraphics[width=.33\linewidth]{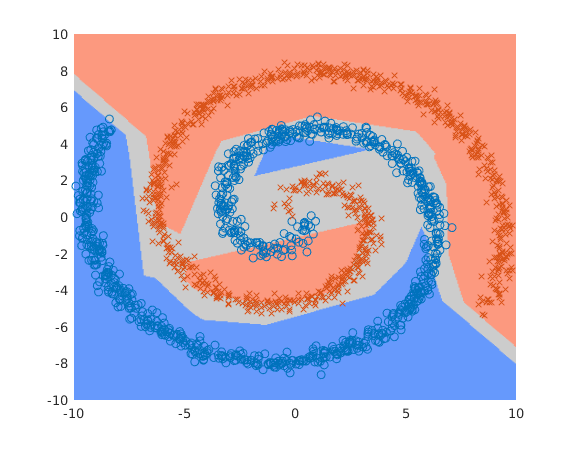}
    \includegraphics[width=.33\linewidth]{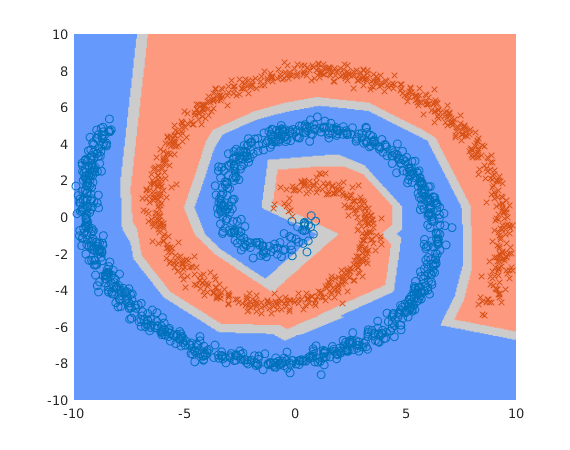}
    \caption{2-D spiral data and regions certified against $L_{\infty}$ perturbations up to 0.1 for best ReLU network (left), SDP network (center), and GroupSort network (right).}
    \label{fig:spiral}
\end{figure}

\begin{wrapfigure}{R}{0.5\textwidth}
    \centering
    \includegraphics[width=0.5\textwidth]{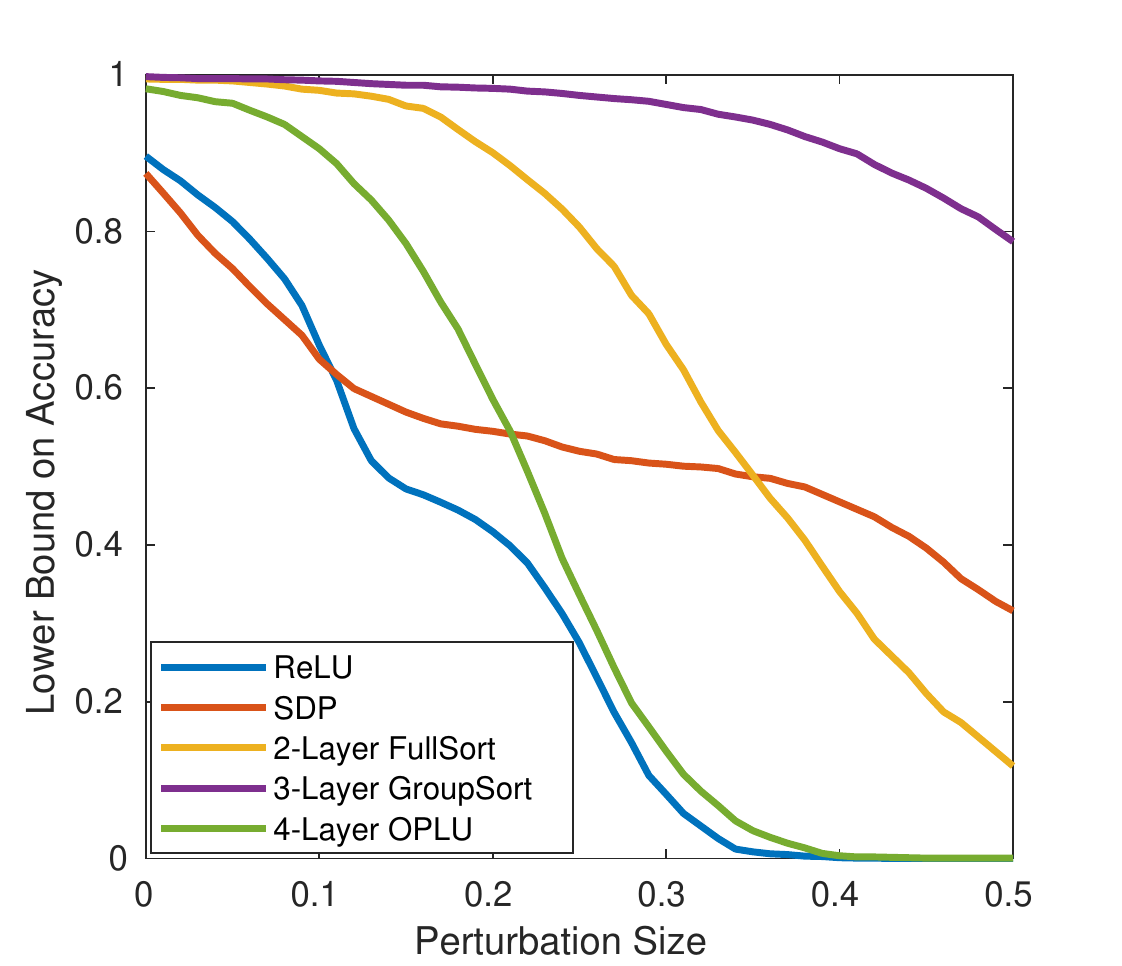}
    \caption{Comparison bounds provided by \newline different certification methods on 2-D spiral data.}
    \label{fig:spiral_plots}
\end{wrapfigure}

We evaluated the algorithms on two fairly simple datasets. First, we have the classic "2-D spiral" dataset. Figures \ref{fig:spiral} and \ref{fig:spiral_plots} respectively show the visual and quantitative results of the different algorithms on this dataset, using the $L_{\infty}$ norm. We also empirically measured the quality of the Lipschitz bound by sampling the space using a fine grid and returning the largest gradient found. As shown in Table \ref{tab:tightness}, using shallow GroupSort, we were able to produce a reasonably tight bound with high accuracy on normal data, while the other methods were off by at least a factor of two. 

We can see that there is a benefit from both multiple layers and larger group sizes. Our intuitive sense is that both the size and number of groups are important: larger groups allow more complex functions to be represented, while having more groups makes optimization easier. Similar to traditional deep learning, depth also appears to be helpful in representing and finding good, compact approximations, at least on this dataset.

\begin{table}
\caption{Certified vs. empirical Lipschitz constant}
\begin{center}\label{tab:tightness}
\begin{tabular}{c c c c}
\hline
Algorithm & Lipschitz Bound & Largest Gradient & Ratio \\
\hline
ReLU  &     9.22& 2.91& 3.17 \\
SDP  &   2.91& 1.30& 2.24 \\
2-Layer FullSort  &  5.85& 4.16& 1.41\\
3-Layer GroupSort-10  &  6.36& 5.29& \textbf{1.20}\\
4-Layer OPLU  &  10.77& 3.90& 2.76\\
\hline
\end{tabular}
\end{center}
\end{table}

We also evaluated on MNIST with both $L_{\infty}$ and $L_2$ norms. In our MNIST experiments, two layer GroupSort-10 networks produced the best results for $L_{\infty}$ perturbations, while four layer OPLU networks worked the best for $L_2$. We compare the bounds given by our GroupSort networks to the other methods in Figure \ref{fig:mnist_plots}. Our approach beat the LMT and ReLU-based networks by a wide margin in both norms. We also beat the SDP method in $L_{\infty}$, a feat of note due to the simpler and more general architecture that we employ.

\begin{figure}[t]
    \includegraphics[width=.5\linewidth]{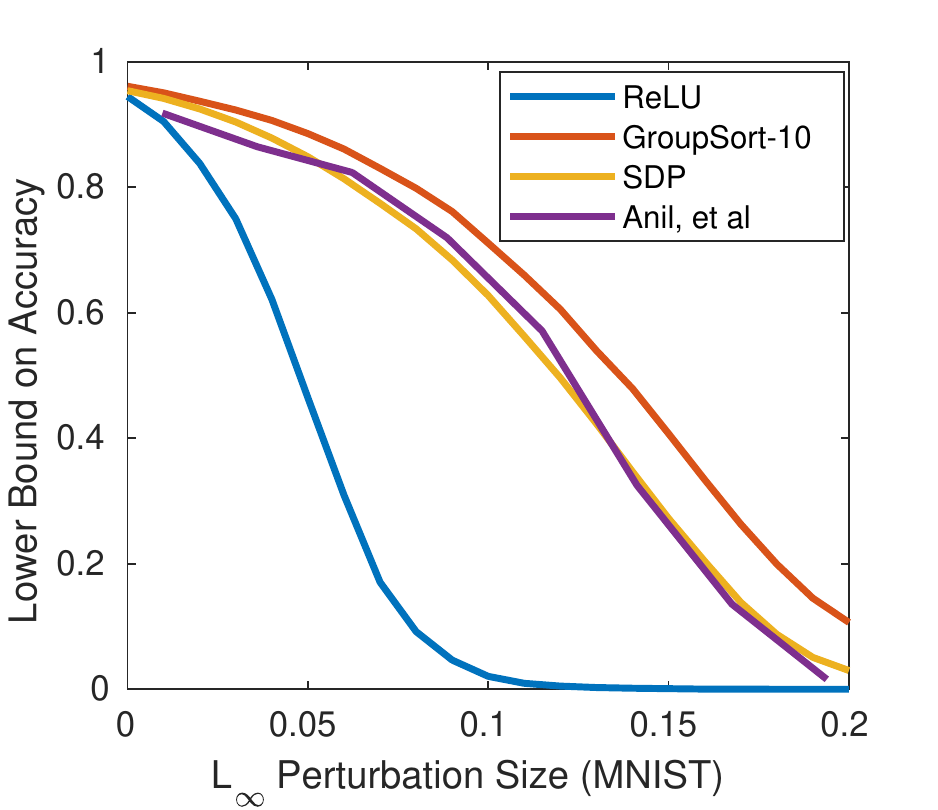}
    \includegraphics[width=.5\linewidth]{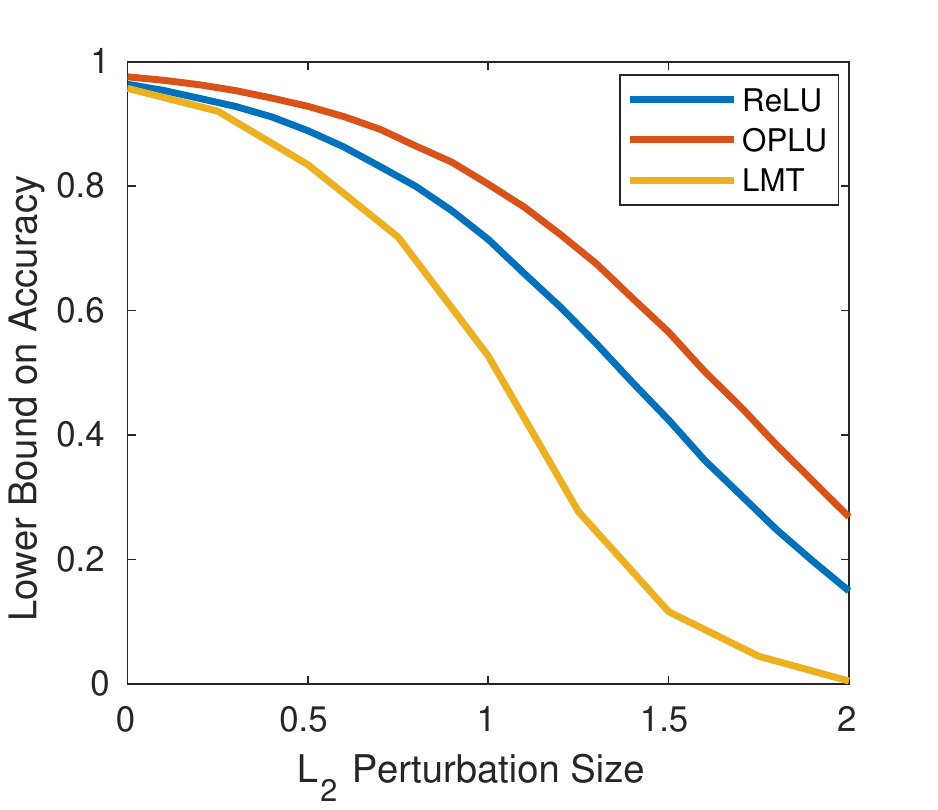}
    \caption{Comparison of bounds provided by different certification methods on MNIST against $L_{\infty}$ and $L_2$ perturbations.}
    \label{fig:mnist_plots}
\end{figure}

It is somewhat disappointing that the two layer network currently gives the best performance for $L_{\infty}$, and that it is not yet competitive with the multi-layered, certified approach of \cite{Wong2018}. In principle, because they are ULAs, three layer GroupSort networks can produce much stronger certificates on the training data than we are currently seeing. That combined with the fact that we have no indication of overfitting, seems to imply that optimization is our primary difficulty when moving from simple 2-D spirals to more complex data.

%For the $L_{inf}$ experiment, we used a smoothed version of the $L_1$ operator norm in the objective function to help training stability: 
%\begin{eqnarray}
%L = H(f(x),y) + \lambda \prod_{i=1}^{k}\|W_i\|_p
%\end{eqnarray}
%Importantly, 

\section{Conclusions}
Our results show that

\begin{enumerate}
    \item Three layer Lipschitz bounded FullSort networks are ULAs that perform well on both a 2-D spiral dataset and MNIST.
    \item GroupSort networks are capable of providing certified defenses while attaining competitive performance on clean data.
    \item ULAs, particularly sorting activation functions, warrant more study.
\end{enumerate}

One of the major challenges we faced in the application of GroupSort networks is that they are routinely difficult to train. In fact, in all of our experiments we saw \emph{no} evidence of overfitting. While there have been years of AI research providing the methods and intuition behind how to train networks with monotonically increasing activation functions, GroupSort is fairly new and relatively unstudied. We are hopeful that through more time and research, methods will arise that improve the training of these nets in order to construct larger, more expressive networks.

We have improved on our understanding of ULAs from $\mathbb{R}^n$ into $\mathbb{R}^m$ under the $L_\infty$ norm. There are many more norms that $\mathbb{R}^n$ can be imbued with. For example, we still do not know if there is a ULA for functions mapping between these spaces under the $L_2$ norm, not-to-mention general $L_p$ norms. Hence, there are still a lot of open questions about the existence of ULAs under certain conditions.

\section*{Acknowledgments}
Thank you to all the people who helped us bring this paper to print. In particular, 
thank you to Noah Marcus for enormously helpful proofreading.

\bibliography{ULA_bib}{}
\bibliographystyle{plain}
\end{document}